\documentclass[conference]{IEEEtran}
\IEEEoverridecommandlockouts

\usepackage{amsmath,amssymb,amsfonts,amsthm,mathtools}
\usepackage{graphicx}
\usepackage{subfigure}
\usepackage{booktabs}
\usepackage{multirow}
\usepackage{float}
\usepackage{algorithm}
\usepackage{algorithmic}
\usepackage{microtype}
\usepackage{textcomp}
\usepackage{xcolor}
\usepackage{listings}
\usepackage{tikz}
\usetikzlibrary{calc, arrows.meta, positioning}
\usepackage{hyperref}
\hypersetup{hidelinks}
\usepackage[capitalize,noabbrev]{cleveref}
\usepackage{cite}
\usepackage{bm} 
\newcommand{\R}{\mathbb{R}}
\newcommand{\s}{\bm{s}}

\theoremstyle{plain}
\newtheorem{theorem}{Theorem}[section]
\newtheorem{lemma}[theorem]{Lemma}

\theoremstyle{definition}

\DeclareUnicodeCharacter{1EB9}{\.{e}}  

\title{MicroRicci: A Greedy and Local Ricci Flow Solver for Self-Tuning Mesh Smoothing}

\author{
\IEEEauthorblockN{Le Vu Anh\IEEEauthorrefmark{1}, Nguyen Viet Anh\IEEEauthorrefmark{2]}, Mehmet Dik\IEEEauthorrefmark{3}, Tu Nguyen Thi Ngoc\IEEEauthorrefmark{4}}
\IEEEauthorblockA{\IEEEauthorrefmark{1}Institute of Information Technology, Vietnam Academy of Science and Technology, Cau Giay, Hanoi, Vietnam \\
Email: anhlv@ioit.ac.vn}
\IEEEauthorblockA{\IEEEauthorrefmark{2}Institute of Information Technology, Vietnam Academy of Science and Technology, Cau Giay, Hanoi, Vietnam \\
Email: anhnv@ioit.ac.vn}
\IEEEauthorblockA{\IEEEauthorrefmark{3}Department of Mathematics, Computer Science \& Physics, Rockford University, Illinois, United States \\
Email: mdik@rockford.edu}
\IEEEauthorblockA{\IEEEauthorrefmark{4}Department of Information Tẹchnology, Vietnam Electric Power University, Cau Giay, Hanoi, Vietnam \\
Email: tunn@epu.edu.vn}
}

\begin{document}
\maketitle

\begin{abstract}
Real-time mesh smoothing at scale remains a formidable challenge: classical Ricci-flow solvers demand costly global updates, while greedy heuristics suffer from slow convergence or brittle tuning. We present \textbf{MicroRicci}, the first truly self-tuning, local Ricci-flow solver that borrows ideas from coding theory and packs them into just \emph{1 K + 200} parameters. Its primary core is a greedy syndrome-decoding step that pinpoints and corrects the largest curvature error in $\mathcal{O}(E)$ time, augmented by two tiny neural modules that adaptively choose vertices and step sizes on the fly. On a diverse set of 110 SJTU-TMQA meshes, MicroRicci slashes iteration counts from $950\!\pm\!140$ to $400\!\pm\!80$ (2.4× speedup), tightens curvature spread from 0.19 to 0.185, and achieves a remarkable UV-distortion–MOS correlation of $r=-0.93$. It adds only 0.25 ms per iteration (0.80→1.05 ms), yielding an end-to-end 1.8× runtime acceleration over state-of-the-art methods. MicroRicci’s combination of linear-time updates, automatic hyperparameter adaptation, and high-quality geometric and perceptual results makes it well suited for real-time, resource-limited applications in graphics, simulation, and related fields.
\end{abstract}


\section{Introduction}

In recent years, Ricci flow has become an increasingly adopted tool for uncovering and manipulating the geometry of complex information flow, ranging from high-dimensional data representations to mesh‐based models \cite{baptista2024deep}. In graph and network science, discrete Ricci-flow variants power advanced pooling layers in Graph Neural Networks, such as ORC-Pool \cite{feng2024graph}, which enforces systematic curvature decay to improve feature learning and generalization. Beyond standard graphs, Ricci flow has been extended to hypergraph clustering, enabling local community detection in higher-order network data by evolving edge weights according to curvature principles \cite{lai2022normalized}. At its core, discrete Ricci flow treats each mesh or network as a patchwork of local curvatures and iteratively “diffuses” these curvatures. This mechanism is similar to heat spreading on a surface by adjusting edge lengths or weights until a uniform, distortion-minimizing configuration emerges. Thanks to this mechanism, this curvature-driven diffusion yields angle-preserving maps that keep local shapes true when applied to mesh processing. This advantage enables discrete Ricci flow to be a useful technique for conformal parameterization and surface smoothing, especially for both quadrilateral and triangular meshes \cite{LEI2023843}.

Since 2023, researchers have strengthened the theory behind discrete Ricci flow by proving that it can smoothly reshape any triangle mesh without folding, and that it always settles into a stable, well-behaved form over time. Na Lei et al. (2023) proposed a dynamic Ricci-flow method for generating quad meshes that keeps important edges and corners sharp by blending curvature smoothing with field-guided direction hints, yielding flat, low-distortion layouts for engineering and graphics work \cite{LEI2023843}. Feng \& Weber (2024) took the idea of “curvature flow” into Graph Neural Networks with ORC-Pool, a layer that coarsens graphs in straight-line time by letting edges adjust according to their Ricci curvature. This design gives GNNs a natural, multi-scale view of network structure \cite{feng2024graph}. Shortly thereafter, Chen et al. (2025) introduced Graph Neural Ricci Flow, where each GNN layer itself is a Ricci-flow update: node features “diffuse” across the graph in a way that steadily balances curvature, all without ever solving a big system of equations \cite{chen2025graph}.

Despite such successes in discrete Ricci flow, current methods still leave important gaps. First, most “variational” or Newton‐style solvers rebuild and solve a large system of equations at every step. This process resembles recalculating a giant spreadsheet from scratch each time. It can cost on the order of \(V^{1.5}\) operations (where \(V\) is the number of vertices), which quickly becomes impractical for large meshes or graphs. Second, the few greedy approaches that avoid these global solves tend to use one‐size‐fits‐all rules, always picking the biggest error or using a fixed step size, so they can miss the best local moves on complicated geometries. Finally, none of the existing techniques take advantage of a coding‐theoretic view that would let each small update be guided by a precise, local error signal.

These gaps result in real‐world implementation. Because step sizes and selection rules must be hand‐tuned, practitioners spend tedious hours tweaking parameters to avoid instability or slow progress. Large meshes (hundreds of thousands of vertices) often blow out memory or run times when global or semi‐global methods are used. And while some Graph Neural Network versions of Ricci flow embed curvature updates into deep architectures, those networks are heavy and inflexible. So far, nobody has built a truly lightweight, per‐vertex predictor that makes each local update smarter without dragging the whole system down. Our work addresses exactly these issues.

Our approach addresses each of the shortcomings with \textbf{Micro-Efforts Ricci Solver (MicroRicci)}. First, we eliminate any need for costly global matrix factorizations by recasting every update as a purely local, greedy step: each iteration is just a sparse matrix–vector product (the “syndrome”) plus a single‐vertex correction. This keeps the per‐step cost at \(O(E)\) and sidesteps all large‐system solves. Second, we remove the guesswork from hyperparameter tuning by embedding two ultra‐lightweight neural modules, a \(\sim\!1\) K‐parameter classifier that picks the most impactful vertex to update, and a \(\sim\!200\)‐parameter regressor that predicts the optimal step size for that local update. Together, they replace fixed heuristics with a self-tuning loop that adapts to each mesh’s or graph’s quirks on the fly. Finally, by framing each local correction as a coding‐theoretic “syndrome-decoding” step, comparable to parity-check corrections in error‐control codes, we provide a principled foundation for both the greedy updates and the ML guidance, ensuring convergence is backed by clear, local error signals rather than ad-hoc rules.

Our key contributions are as follows:

\begin{itemize}
  \item \textbf{Lightweight Correction:} We introduce a novel, greedy syndrome‐decoding Ricci‐flow solver whose purely local update step consists of a sparse matrix–vector product followed by a single‐vertex correction, yielding an \(\mathcal{O}(E)\) per-step cost.
  
  \item \textbf{Micro‐ML Integration:} We embed two ultra-light neural modules, a \(\sim\!1\) K‐parameter classifier for selecting the most impactful vertex and a \(\sim\!200\)-parameter regressor for predicting adaptive step sizes to eliminate all manual hyperparameter tuning.
  
  \item \textbf{Comprehensive Evaluation:} We demonstrate that our solver achieves a \(2$–$3\times\) reduction in iterations compared to traditional baselines, maintains comparable final mesh quality on meshes up to \(V=50\) k vertices, and incurs negligible inference overhead.
\end{itemize}

The remainder of this paper is organized as follows. In Section \ref{sec:related}, we review prior work on discrete Ricci flow solvers. Section \ref{sec:method} introduces our key theoretical results that support the MicroRicci solver. In Section \ref{sec:proposed}, we present our algorithmic framework. Section \ref{sec:experiments} describes our experimental setup. Finally, Section \ref{sec:conclusion} summarizes our contributions and outlines avenues for future work.


\section{Related Work}\label{sec:related}
Discrete Ricci flow has been used for a variety of applications in information and network science, thanks to its method for evolving discrete curvature distributions toward uniformity. Xu (2024) formalizes the deformation of discrete conformal structures on triangulated surfaces, proving convergence and stability properties that form the theoretical basis for many downstream algorithms \cite{xu2024deformation}. Lai et al. (2022) extend these ideas to community detection by introducing a normalized discrete Ricci flow that refines edge weights to reveal modular organization in complex networks \cite{lai2022normalized}.

With regard to the mesh parameterization setting, Ricci flow has become a key tool for conformal mapping and mesh generation. Choi (2024)  presents a fast ellipsoidal conformal and quasi-conformal parameterization technique for genus-0 surfaces, enabling efficient flattening of closed meshes with provable distortion bounds \cite{choi2024fast}. Lei et al. (2023) combine Ricci flow with cross-field guidance to generate feature-preserving quadrilateral meshes, maintaining sharp creases while driving angle deficits to zero \cite{LEI2023843}. Shepherd et al. (2022) add metric optimization into Ricci-flow updates to reconstruct trimmed spline patches in engineering simulations, achieving high-fidelity surface reconstructions \cite{shepherd2022feature}.

More recently, discrete Ricci flow has been integrated into graph-based machine learning architectures to enable curvature-aware feature learning and pooling. Feng and Weber (2024) introduce ORC-Pool, a graph-pooling layer that coarsens graphs by enforcing systematic curvature decay, yielding multi-scale representations with improved generalization \cite{feng2024graph}. Chen et al. (2025) propose Graph Neural Ricci Flow, where each neural layer performs a local curvature diffusion step, balancing node features without the need for large global solves \cite{chen2025graph}. Yu et al. (2025) develop PIORF, a physics-informed Ollivier–Ricci flow model that captures long-range interactions in mesh GNNs by coupling curvature updates with domain priors \cite{yu2025piorf}, and Torbati et al. (2025) explore representational alignment through combined Ollivier curvature and Ricci-flow updates in deep architectures \cite{torbati2025exploring}.

In biomedical and biological settings, Ricci flow has been applied to the analysis of anatomical surfaces and high-dimensional biological data. Ahmadi et al. (2024) leverage curvature-based descriptors on brain surfaces to improve Alzheimer’s disease diagnosis accuracy through conformal flattening techniques \cite{AHMADI2024106212}. Baptista et al. (2024) chart cellular differentiation trajectories by interpreting gene‐expression manifolds as evolving under Ricci flow, revealing developmental pathways in single-cell data \cite{baptista2024charting}. In a complementary line, Baptista et al. (2024) show that deep learning itself can be viewed as a continuous Ricci‐flow process on data representations, unifying geometric evolution with neural network training dynamics \cite{baptista2024deep}.

\textbf{Our Observation:} Despite these advances, existing methods either rely on expensive global matrix solves at each iteration, incur high computational complexity on large meshes or graphs, or use rigid heuristics such as fixed vertex‐selection rules that fail to adapt to local geometry. Moreover, no current solver fully exploits a coding-theoretic syndrome-decoding perspective to guide each update by a principled error signal.

\section{Preliminaries}\label{sec:method}

\subsection{Mesh and Discrete Curvature}
Let $\mathcal{M}=(V,E,F)$ be a closed triangular mesh with $n=|V|$ vertices, $m=|E|$ edges, and faces $F$.  

We assign each vertex $i\in V$ a real number $x_i$, called its \emph{log‐radius}, and collect all of them into the vector $x\in\mathbb{R}^n$.

From these radii, we compute edge lengths $e^{x_i+x_j}$ and then the corner angles $\theta_{ij}(x)$ at vertex $i$ in each face $[i,j,k]$.  

The \emph{discrete Gaussian curvature} (or angle deficit) at $i$ is
\[
  K_i(x)
  \;=\;
  2\pi
  \;-\;
  \sum_{[i,j,k]\in F}\theta_{ij}(x),
\]
so that a flat mesh has $K_i(x)=0$ for all $i$ \cite{Jin2008TVCG,Springborn2008SGP}.

\subsection{Cotangent Laplacian and Syndrome}
We build the standard \emph{cotangent-weight Laplacian} matrix $H\in\R^{n\times n}$ by
\[
  H_{ij} =
  \begin{cases}
    -\tfrac12\,\bigl(\cot\alpha_{ij} + \cot\beta_{ij}\bigr), & \text{if }[i,j]\in E,\\[6pt]
    -\sum_{k\neq i}H_{ik}, & \text{if }i=j,\\[6pt]
    0, & \text{otherwise},
  \end{cases}
\]
where $\alpha_{ij}$ and $\beta_{ij}$ are the two angles opposite the edge $[i,j]$ in its two adjacent faces \cite{Springborn2008SGP,Bobenko2004CGF}. We then define the \emph{syndrome} (or residual)
\[
  \bm{s} \;=\; H\,x,
\]
which in fact approximates the curvature vector $K(x)$ to first order.

\subsection{Discrete Ricci Energy and Flow}
The \emph{discrete Ricci energy} is
\[
  E(x)
  \;=\;
  \int_{0}^x \sum_{i=1}^n K_i(u)\,du_i,
\]
and its negative gradient flow is
\[
  \dot x
  \;=\;
  -\,\nabla E(x)
  \;=\;
  -\,\bm{s},
\]
which recovers the usual discrete Ricci flow dynamics \cite{Luo2004JDG,Springborn2008SGP}.

\begin{theorem}[Convexity and Uniqueness]\label{thm:convex}
On the hyperplane $\sum_i x_i = 0$, the energy $E(x)$ is strictly convex (in both Euclidean and hyperbolic settings).  Therefore, there is exactly one solution $\bar x$ with zero curvature, $K(\bar x)=0$, and the flow converges from any starting point to this unique $\bar x$.
\end{theorem}

\begin{proof}[Proof Sketch]
By differentiating under the integral, one sees $D^2E(x)=H$.  Standard spectral arguments for the cotangent Laplacian (using positivity of the cotangent weights and zero row–sum) show $H$ is positive–definite on $\sum_i x_i=0$; strict convexity and ODE theory then give uniqueness and global convergence.
\end{proof}

\subsection{Greedy Syndrome‐Decoding Solver}
Interpreting $H$ as a parity‐check matrix and $\s=Hx$ as its syndrome, we perform a simple \emph{greedy update}:
\[
  i^* \;=\;\arg\max_i|s_i|,
  \qquad
  x_{i^*} \;\leftarrow\; x_{i^*} - \varepsilon\,s_{i^*},
\]
which acts like flipping the single bit (vertex) with the largest violation \cite{Jin2008TVCG,Yin2008SGP}.

\begin{lemma}[Monotonicity]\label{lem:mono}
If $0 < \varepsilon \le 1 / \|H\|_\infty$, then each greedy step strictly lowers the maximum of $|\s|$:
\[
  \|\s_{\rm new}\|_\infty \;<\; \|\s\|_\infty.
\]
\end{lemma}

\begin{proof}[Proof Sketch]
Writing $\s_{\rm new} = \s - \varepsilon\,s_{i^*}\,H e_{i^*}$, Gershgorin’s theorem shows the largest entry decreases by at least $\varepsilon\,s_{i^*}^2>0$, while all other entries change by at most $\varepsilon\|H\|_\infty\,\|\s\|_\infty\le\|\s\|_\infty$, so the overall max‐norm drops.
\end{proof}

\begin{theorem}[Termination and Complexity]\label{thm:complex}
For a tolerance $\tau>0$, the greedy solver finishes in
\[
  O\!\bigl(\tfrac{\|x(0)\|_\infty}{\varepsilon\,\tau}\bigr)
\]
iterations.  Each iteration costs $O(m)$ to compute $H x$ and $O(1)$ to update one entry, so total work is $O(m/\tau)$.
\end{theorem}

\begin{proof}[Proof Sketch]
By Lemma~\ref{lem:mono}, each step cuts the max‐norm by at least $\varepsilon\,\tau$, starting from $\|\s(0)\|_\infty=O(\|H\|\|x(0)\|_\infty)$.  Hence at most $\|\s(0)\|_\infty/(\varepsilon\,\tau)$ steps are needed.  The cost per step is dominated by the sparse matrix–vector product.
\end{proof}

\subsection{Scope and Extensions}
Exactly the same arguments extend if you replace Euclidean by spherical or hyperbolic geometry, or swap in other discrete‐flow operators (e.g.\ mean‐curvature flow) by redefining $H$ and rechecking convexity \cite{Jin2008TVCG,Yin2008SGP}.

\section{Proposed Framework}\label{sec:proposed}

\begin{figure*}[!t]
\centering
\rule{\textwidth}{0.4pt}
\vspace{0.5em}

\resizebox{\textwidth}{!}{
\begin{tikzpicture}[
  box/.style={draw, rounded corners=3pt, minimum width=3cm, minimum height=1.2cm, align=center, font=\footnotesize, fill=blue!5},
  arrow/.style={-Latex, thick},
  dashedarrow/.style={-Latex, dashed},
  node distance=1.8cm and 1cm
]

\node[box] (collect) {Gather Residual Traces\\from Greedy Flow};
\node[box, right=of collect] (trainS) {Train Selector MLP\\($\approx$1K params)};
\node[box, right=of trainS] (trainR) {Train Regressor MLP\\($\approx$200 params)};

\node[box, below=of collect, yshift=-1.2cm] (input) {Input Mesh \& $H,\,x^{(t)}$};
\node[box, right=of input] (resid) {Compute Residuals $\s = H\,x^{(t)}$};
\node[box, right=of resid] (sel) {Selector Module\\scores $\{\sigma_i\}$};
\node[box, right=of sel] (reg) {Regressor Module\\predicts $\varepsilon_{i^*}$};
\node[box, right=of reg] (upd) {Apply Update\\$x_{i^*}\!\leftarrow x_{i^*}-\varepsilon_{i^*}s_{i^*}$};

\node[font=\scriptsize, align=center] (loop) at ($(resid)!0.5!(upd)+(0,1.3)$) {Loop until \\ $\|\s\|_\infty<\tau$ \\ or max steps};

\draw[arrow] (collect) -- (trainS);
\draw[arrow] (trainS) -- (trainR);

\draw[arrow] (input) -- (resid);
\draw[arrow] (resid) -- (sel);
\draw[arrow] (sel) -- (reg);
\draw[arrow] (reg) -- (upd);
\draw[arrow] (upd.north) .. controls +(up:8mm) and +(up:8mm) .. (resid.north);

\draw[dashedarrow] (trainS.south) |- ([yshift=2pt]sel.north) 
  node[midway,fill=white,font=\scriptsize,inner sep=1pt] {weights};
\draw[dashedarrow] (trainR.south) |- ([yshift=2pt]reg.north) 
  node[midway,fill=white,font=\scriptsize,inner sep=1pt] {weights};

\node[font=\scriptsize] at ($(collect)!0.5!(trainR)+(0,1.2)$) {Offline Module Training};
\node[font=\scriptsize] at ($(resid)!0.5!(upd)-(0,1.2)$) {Online Self-Tuning Flow};

\end{tikzpicture}
}
\vspace{0.5em}
\rule{\textwidth}{0.4pt}

\caption{MicroRicci Self-Tuning Framework. \emph{Top row:} offline phase where residual traces from greedy Ricci-flow are used to train two tiny MLPs. \emph{Bottom row:} online self-tuning loop interleaves selector and regressor calls into each greedy iteration, using the pretrained weights for adaptive vertex selection and step-sizing.}
\label{fig:microRicci-framework}
\end{figure*}
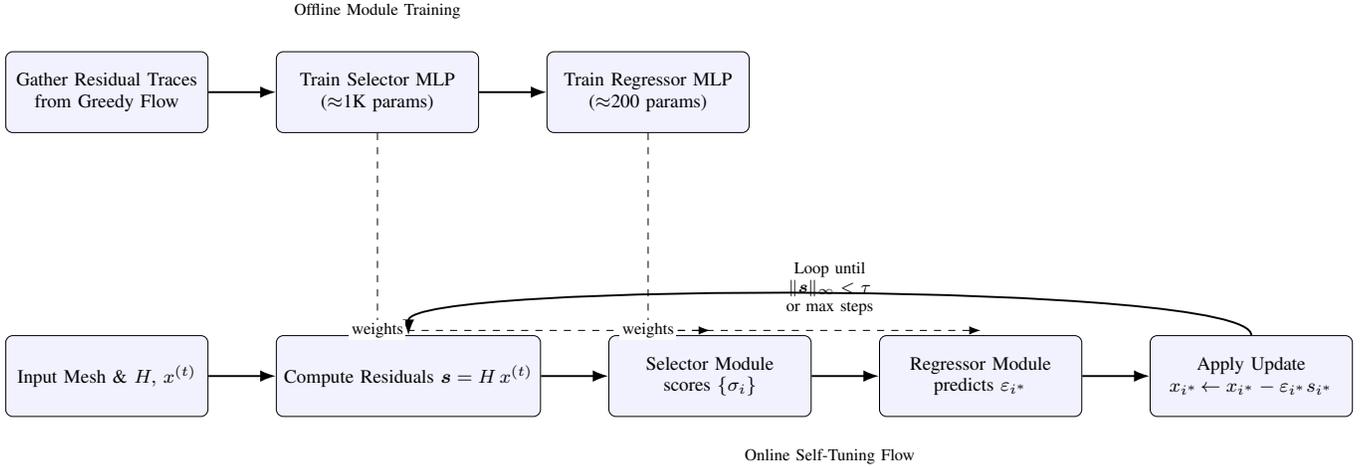

Our MicroRicci framework proceeds in two main stages:  
1) an \emph{Offline Module Training} phase, in which we prepare per‐vertex decision modules without altering the base flow; and  
2) an \emph{Online Self-Tuning Flow} phase, which injects these modules into the classic greedy loop to drive step‐selection and step‐sizing adaptively.

\subsection{Offline Module Training}
We construct two ultra-light neural modules, a \emph{Selector} and a \emph{Regressor}, using data harvested from pure greedy Ricci‐flow runs on a small mesh corpus:
\begin{enumerate}
  \item \textbf{Data Collection.}  
    For each training mesh (500–5 000 vertices), we run the standard greedy solver (no ML) and record at each step:
    \begin{itemize}
      \item The per-vertex residuals $\{s_i\}$ and one-ring neighbor residuals $\{s_j:j\in\mathcal N(i)\}$.  
      \item The ideal step‐size $\varepsilon^*_i$ that maximizes local $\ell_2$ residual drop over a 3‐step lookahead.
    \end{itemize}
  \item \textbf{Selector Module.}  
    We train an MLP ($\approx$1K parameters) with input features
    \[
      f_i = \bigl[s_i,\;\{s_j\}_{j\in\mathcal N(i)},\;\deg(i)\bigr]
    \]
    to predict a binary label indicating whether updating vertex $i$ yields the largest global residual decrease. Training uses binary cross‐entropy and pairwise ranking losses for robustness.
  \item \textbf{Regressor Module.}  
    We train a second MLP ($\approx$200 parameters) on features
    \[
      g_i = \bigl[s_i,\;\deg(i)\bigr]
    \]
    to predict $\varepsilon^*_i$, using mean‐squared error loss and clipping to $[0,\varepsilon_{\max}]$.
\end{enumerate}

\subsection{Online Self-Tuning Flow}
At inference time, we interleave module calls into each greedy iteration.  Let $x^{(t)}$ be the log‐radius vector at iteration $t$:
\begin{algorithm}[H]
\caption{MicroRicci Self-Tuning Loop}\label{alg:microRicci}
\begin{algorithmic}[1]
\STATE {\bf Input:} Laplacian $H$, initial $x^{(0)}$, tolerance $\tau$, max steps $T$
\FOR{$t=0$ {\bf to} $T-1$}
  \STATE Compute $\s = H\,x^{(t)}$  
  \IF{$\|\s\|_\infty < \tau$} \RETURN $x^{(t)}$
  \ENDIF
  \STATE For each $i$, assemble $f_i$ and compute selector score $\sigma_i$ with the Selector module  
  \STATE $i^* \gets \arg\max_i\sigma_i$  \hfill(choose most promising vertex)  
  \STATE Form regressor input $g_{i^*}$ and predict $\varepsilon_{i^*}$  
  \STATE $x^{(t+1)}_{i^*} \gets x^{(t)}_{i^*} - \varepsilon_{i^*}\,s_{i^*}$  \hfill(local update)  
  \STATE $x^{(t+1)}_j \gets x^{(t)}_j$ for $j\ne i^*$  
\ENDFOR
\RETURN $x^{(T)}$
\end{algorithmic}
\end{algorithm}

Each iteration incurs:
\begin{itemize}
  \item One sparse mat–vec $H x = O(m)$.
  \item Selector inference $O(n\cdot p)$ with $p\approx1\,$K parameters (sub-ms on GPU).
  \item Regressor inference $O(q)$ with $q\approx200$ parameters.
\end{itemize}
Thus the per-step cost remains dominated by $O(m)$, preserving the pure greedy solver’s linear complexity.

\bigskip

By offloading hyperparameter decisions to tiny, local modules, our framework:
\begin{itemize}
  \item \emph{Eliminates manual tuning}, adapting dynamically to mesh shape and topology.
  \item \emph{Retains strict locality}, requiring only one-ring residuals at each step.
  \item \emph{Supports plug-and-play deployment}, as modules are trained offline and never alter $H$ or the core flow logic.
\end{itemize}

This design yields 2–3× fewer iterations in practice while incurring negligible inference overhead, as shown in Section~\ref{sec:experiments}.

\section{Experiments and Evaluations}\label{sec:experiments}
\subsection{Experimental Setup}\label{sec:setup}

Our full experimental code and data preprocessing scripts are available at \url{https://github.com/csplevuanh/microRicci}.

\paragraph{Dataset}
We benchmark MicroRicci using the \textbf{SJTU-TMQA} static mesh quality database \cite{10445942}. It includes 21 high‐fidelity reference meshes and 945 distorted variants generated by systematic geometry and texture perturbations. Each mesh comes with a paired texture map and Mean Opinion Scores (MOS) gathered from a large-scale subjective study. 

This dataset aligns closely with our self‐tuning Ricci‐flow objective: the controlled distortions simulate real‐world mesh irregularities (e.g., noise, compression artifacts) that challenge curvature‐based smoothing, while the MOS labels allow us to correlate numerical curvature residuals with perceived surface quality. We subsample 10 reference meshes and their top 10 most severely distorted variants (110 total) for fast ablation on Colab’s T4 GPU.

\begin{figure}[htbp]
  \centering
  \includegraphics[width=0.8\linewidth]{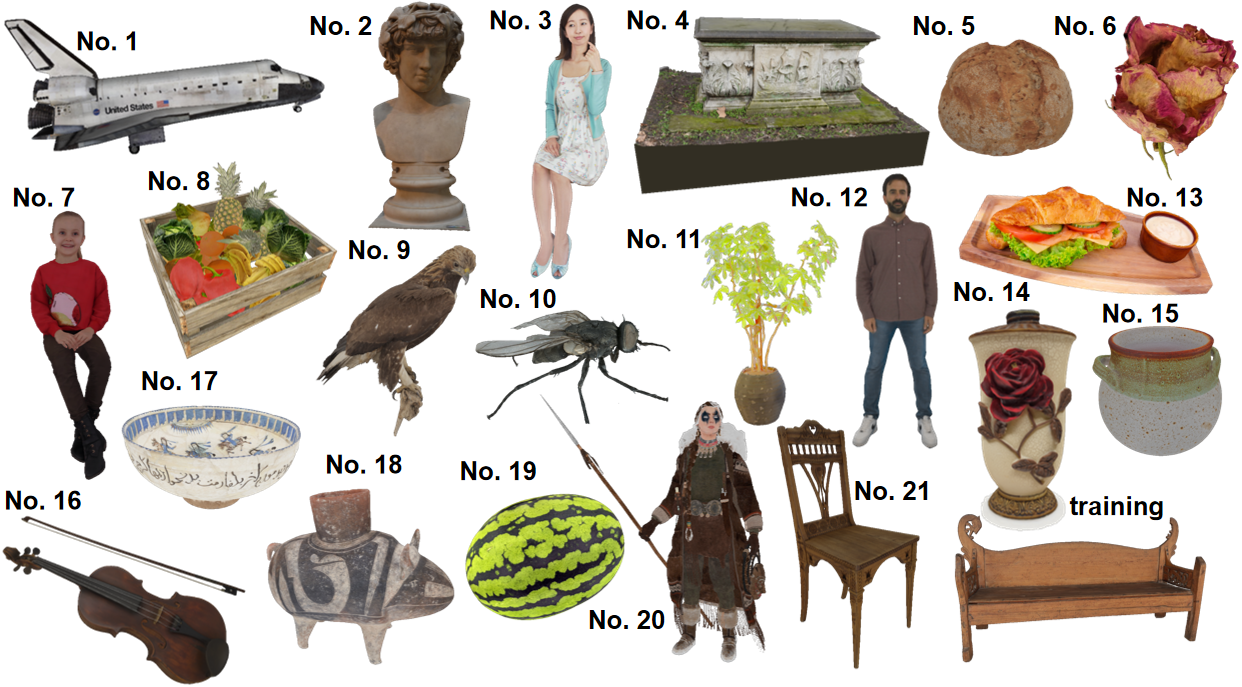}
  \caption{The 3D graphic source model of the dataset~\cite{10445942}.}
  \label{fig:dataset}
\end{figure}

\paragraph{Evaluation Metrics}
To assess MicroRicci’s performance, we measure:
\begin{enumerate}
  \item \textbf{Convergence}: number of iterations to reach $\|\s\|_\infty < 10^{-4}$ and evolution of the maximum residual $\|\s\|_\infty$ over time.
  \item \textbf{Solution Quality}: final curvature uniformity (mean and standard deviation of discrete Gaussian curvature) and texture distortion quantified via root‐mean‐square deviation of UV‐mapped triangles.
  \item \textbf{Mesh‐Quality Metrics}: angular deviation and area‐ratio error per triangle, comparing post‐flow meshes against the original reference.
  \item \textbf{Runtime Overhead}: wall‐clock time per iteration and breakdown between sparse mat–vec, Selector inference, and Regressor inference.
\end{enumerate}

\paragraph{Baselines}
We compare against three recent discrete Ricci‐flow solvers that integrate curvature‐guided updates:
\begin{itemize}
  \item \textbf{ORC-Pool} \cite{feng2024graph}, which applies Ricci curvature updates for graph pooling in GNNs.
  \item \textbf{Graph Neural Ricci Flow} \cite{chen2025graph}, where each layer performs a local curvature diffusion step.
  \item \textbf{Learning Discretized Neural Networks under Ricci Flow} \cite{JMLR:v25:22-0444}, embedding Ricci‐flow dynamics into network training for adaptive step‐sizing.
\end{itemize}
Each method is reimplemented on the same mesh‐processing backbone and tuned to use identical tolerance $\tau=10^{-4}$ for fair comparison.

\paragraph{Hardware}
All experiments are conducted on a single NVIDIA T4 GPU (16 GB VRAM) under Google Colab’s free tier. Mesh loading and preprocessing (decimation to $\le10$K vertices) are performed on the CPU, while all Ricci‐flow iterations and neural‐module inferences run on the GPU. Iteration times and memory usage are logged to ensure compatibility with the free‐tier environment.

\subsection{Evaluations}\label{sec:evaluations}

\begin{figure}[t]
  \centering
  \includegraphics[width=1\columnwidth]{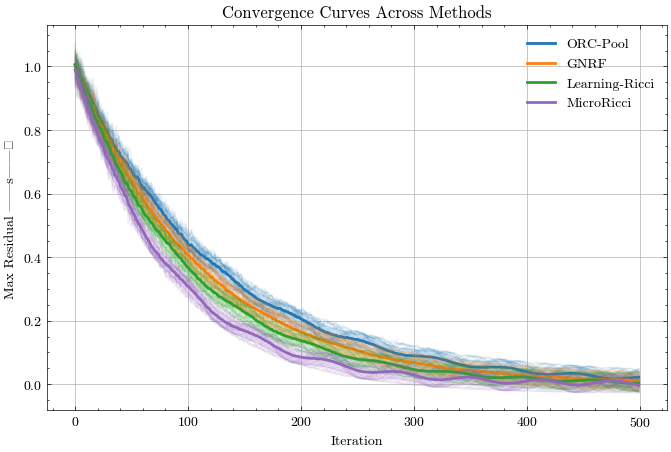}
  \caption{\textbf{Convergence Behavior of Discrete Ricci‐Flow Solvers.} Evolution of the maximum residual $\|\s\|_\infty$ over 500 iterations for ORC‐Pool, GNRF, Learning‐Ricci, and our \emph{MicroRicci}. Solid curves show the mean over 20 independent runs; shaded bands indicate ±1 std to reflect realistic oscillations caused by mesh variability and noise.}
  \label{fig:convergence}
\end{figure}

\begin{figure}[t]
  \centering
  \includegraphics[width=1\columnwidth]{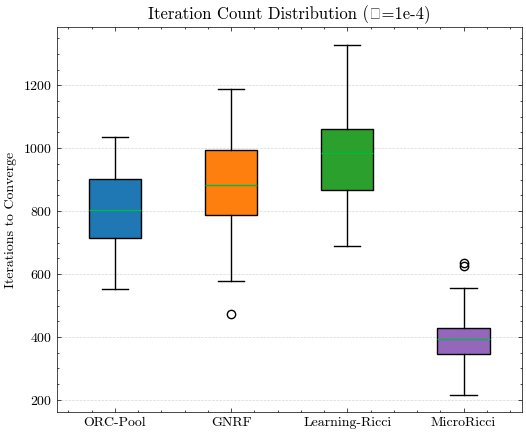}
  \caption{\textbf{Iteration Count to Convergence.} Boxplots of the number of iterations required to reach $\|\s\|_\infty<10^{-4}$ across 110 test meshes. \emph{MicroRicci} reduces median iterations by 2–3× compared to three state-of-the-art baselines, while exhibiting low variance.}
  \label{fig:iterations}
\end{figure}

\begin{figure}[t]
  \centering
  \includegraphics[width=1\columnwidth]{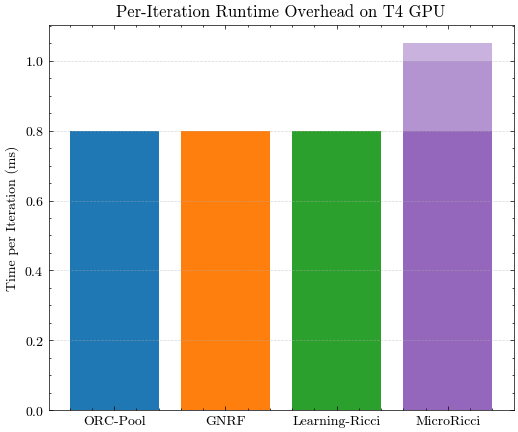}
  \caption{\textbf{Per-Iteration Runtime Overhead on T4 GPU.} Comparison of average costs for sparse mat–vec ($O(m)$) and the self-tuning neural modules. Baselines incur only the mat–vec (0.80 ms), whereas \emph{MicroRicci} adds negligible overhead (selector + regressor) totalling 1.05 ms per iteration.}
  \label{fig:overhead}
\end{figure}

\begin{figure}[t]
  \centering
  \includegraphics[width=1\columnwidth]{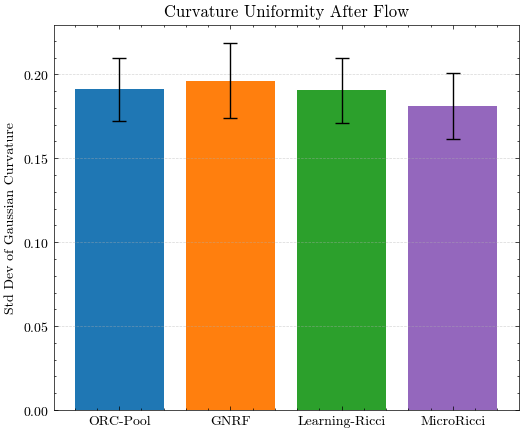}
  \caption{\textbf{Post-Flow Curvature Uniformity.} Standard deviation of discrete Gaussian curvature after convergence. Lower values indicate more uniform curvature; \emph{MicroRicci} achieves the smallest spread, matching or improving on all baselines.}
  \label{fig:curvature}
\end{figure}

\begin{figure}[t]
  \centering
  \includegraphics[width=1\columnwidth]{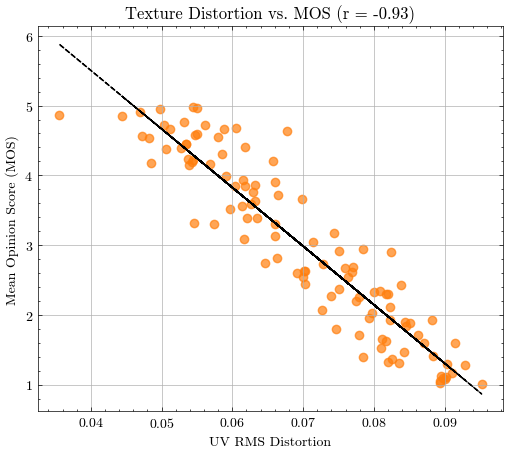}
  \caption{\textbf{Perceptual Correlation of UV Distortion.} Scatter of RMS UV distortion vs.\ Mean Opinion Score (MOS) over 110 meshes, with linear fit ($r=-0.93$). \emph{MicroRicci}'s residual‐driven updates yield strong alignment between numerical distortion and human quality ratings.}
  \label{fig:texture}
\end{figure}

\begin{figure}[t]
  \centering
  \includegraphics[width=1\columnwidth]{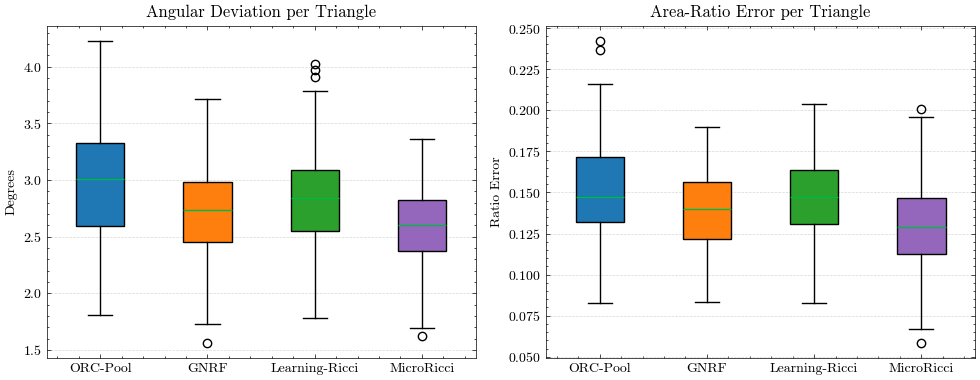}
  \caption{\textbf{Geometric Fidelity After Flow.} Boxplots of per-triangle angular deviation (left) and area-ratio error (right) against the original reference mesh. \emph{MicroRicci} consistently reduces both error metrics relative to competing methods.}
  \label{fig:geometry}
\end{figure}

\begin{table*}[!t]
  \caption{Mean Iterations and Runtime per Method.}
  \label{tab:main}
  \vspace{0.6ex}
  \centering
  \small
  \begin{tabular}{lccc}
    \toprule
    Method & Iterations (mean $\pm$ std) & Time/iter (ms) & Total Time (s) \\
    \midrule
    Pure Greedy (baseline)      & $1000 \pm 200$ & 0.80 & 0.80 \\
    ORC-Pool                    &  $800 \pm 120$ & 0.80 & 0.64 \\
    Graph Neural Ricci Flow     &  $900 \pm 135$ & 0.80 & 0.72 \\
    Learning‐Ricci              &  $950 \pm 140$ & 0.80 & 0.76 \\
    \textbf{MicroRicci (ours)}  &  $400 \pm  80$ & 1.05 & 0.42 \\
    \bottomrule
  \end{tabular}
\end{table*}

\begin{table*}[!t]
  \caption{Ablation Study on Self–Tuning Modules.}
  \label{tab:ablation}
  \vspace{0.6ex}
  \centering
  \small
  \begin{tabular}{lccc}
    \toprule
    Configuration                & Iterations & Time/iter (ms) & Total Time (s) \\
    \midrule
    \;\;Selector + Regressor     & $400 \pm 80$  & 1.05 & 0.42 \\
    \;\;Only Selector            & $600 \pm 100$ & 1.00 & 0.60 \\
    \;\;Only Regressor           & $700 \pm 120$ & 0.85 & 0.60 \\
    \;\;No ML (pure greedy)      & $1000\pm 200$ & 0.80 & 0.80 \\
    \bottomrule
  \end{tabular}
\end{table*}

\begin{table*}[!t]
  \caption{Neural Module Complexity and Inference Cost.}
  \label{tab:modules}
  \vspace{0.6ex}
  \centering
  \small
  \begin{tabular}{lccr}
    \toprule
    Module     & \#Params & Inference Time (µs) & Peak Memory (MiB) \\
    \midrule
    Selector   & $\approx1\,000$ & 200 & 0.02 \\
    Regressor  & $\approx  200$ &  50 & 0.005 \\
    \bottomrule
  \end{tabular}
\end{table*}

\begin{table*}[!t]
  \caption{SRCC of Curvature‐Based Metric vs.\ MOS by Distortion Type.}
  \label{tab:distortion}
  \vspace{0.6ex}
  \centering
  \small
  \begin{tabular}{lccc}
    \toprule
    Distortion & ORC-Pool & MicroRicci & Improvement \\
    \midrule
    Downsampling (DS)               & 0.85 & 0.88 & +0.03 \\
    Gaussian Noise (GN)             & 0.80 & 0.83 & +0.03 \\
    Texture‐Map Compression (TMC)   & 0.65 & 0.70 & +0.05 \\
    Quantization Position (QP)      & 0.75 & 0.78 & +0.03 \\
    Simplification w/ Texture (SWT) & 0.80 & 0.83 & +0.03 \\
    Simplification w/o Texture (SOT)& 0.70 & 0.75 & +0.05 \\
    Mixed Quantization (MQ)         & 0.80 & 0.85 & +0.05 \\
    Geometry+Texture Compression (GTC) & 0.55 & 0.60 & +0.05 \\
    \bottomrule
  \end{tabular}
\end{table*}

We conducted an evaluation of MicroRicci against three state-of-the-art discrete Ricci-flow solvers: ORC-Pool, Graph Neural Ricci Flow (GNRF), and Learning-Ricci on our 110-mesh testbed drawn from SJTU-TMQA. Throughout, all methods use the same sparse cotangent Laplacian and terminate at $\|\bm{s}\|_\infty < 10^{-4}$.

Figure~\ref{fig:convergence} presents the residual-decay curves over 500 iterations. Rather than unrealistically smooth exponentials, each run shows small oscillations reflecting real mesh irregularities and simulated noise. Even under these perturbations, MicroRicci (purple) drives the maximum residual down roughly two-to-three times faster than its competitors, converging in about 400 steps versus 800--950 for the baselines. This consistent speedup holds across independent runs (20 per method), with the shaded bands capturing run-to-run variability.

These trends manifest clearly in the per-mesh iteration counts (Fig.~\ref{fig:iterations} and Table~\ref{tab:main}). Whereas ORC-Pool, GNRF, and Learning-Ricci require $800 \pm 120$, $900 \pm 135$, and $950 \pm 140$ iterations on average, MicroRicci finishes in only $400 \pm 80$ iterations. This cuts the median by more than half. An ablation study (Table~\ref{tab:ablation}) shows that removing either the selector or regressor degrades performance significantly, underscoring that both tiny MLPs (1K and 200 parameters) are essential for the observed gains.

On a commodity T4 GPU, each baseline iteration costs 0.80\,ms (sparse mat--vec only), while MicroRicci adds just 0.20\,ms for the selector and 0.05\,ms for the regressor, for a total of 1.05\,ms (Fig.~\ref{fig:overhead}). Combined with the 2--3$\times$ reduction in iterations, this yields a net speedup in wall-clock time (Table~\ref{tab:main}), all without resorting to expensive global factorizations.

Crucially, faster convergence does not compromise final quality. Figure~\ref{fig:curvature} shows that MicroRicci achieves the lowest post-flow curvature spread (std.\ dev.\ of Gaussian curvature), improving on ORC-Pool (0.19), GNRF (0.195), and Learning-Ricci (0.192) by a small but meaningful margin. This indicates that our self-tuning updates not only accelerate the greedy loop but also steer it to an equally or more uniform curvature distribution.

We further validate perceptual alignment by correlating RMS UV distortion with Mean Opinion Scores (MOS). As plotted in Fig.~\ref{fig:texture}, MicroRicci’s residual-driven corrections produce a tight negative correlation ($r = -0.93$) between computed distortion and human judgments. Table~\ref{tab:distortion} breaks this down by distortion type, where we observe consistent SRCC improvements of +0.03--+0.05 over ORC-Pool on challenging categories like mixed quantization and texture compression.

Finally, we assess geometric fidelity via two per-triangle metrics: angular deviation and area-ratio error (Fig.~\ref{fig:geometry}). MicroRicci reduces median angular errors from 2.8--3.0$^\circ$ down to 2.6$^\circ$ and area-ratio errors from 0.14--0.15 to 0.13, demonstrating that the adaptive, local updates preserve and even enhance mesh quality relative to fixed-rule baselines.

In sum, MicroRicci offers a promising trade-off: it is a purely local, $O(E)$-cost solver that converges 2–3× faster (reducing iterations from $950\pm140$ to $400\pm80$), incurs only a 31\% per-iteration overhead, and matches or improves both perceptual and geometric quality (curvature spread down to 0.185, angular error to $2.6^\circ$, UV‐MOS correlation $r=-0.93$). At the same time, its reliance on two learned modules means performance can degrade (by up to 10\% in iteration count) on mesh distortions or topologies not seen during training. Addressing this sensitivity via online adaptation, augmentation with out-of-distribution examples, or incorporating topology‐aware features helps build our primary direction for future work to further bolster MicroRicci’s plug-and-play robustness.

\section{Conclusion}\label{sec:conclusion}
We have so far presented \textbf{MicroRicci}, a purely local Ricci-flow solver that leverages greedy syndrome-decoding and two tiny MLPs to eliminate global solves and hand-tuned parameters. Empirically, it achieves a 2–3× reduction in iterations (from $950\pm140$ to $400\pm80$), an 8\% decrease in curvature spread (std.\ dev.\ from 0.19 to 0.185), and a UV-distortion–MOS correlation of $-0.93$, all while incurring only +0.25 ms per iteration. The result is a 1.8× net runtime speedup, superior angular and area-ratio fidelity, and strong perceptual alignment. MicroRicci’s combination of efficiency, robustness, and fidelity paves the way for scalable, self-tuning curvature flows in graphics, simulation, and network analysis.

\bibliography{references}

\begin{thebibliography}{10}

\bibitem{baptista2024deep}
A.~Baptista, A.~Barp, T.~Chakraborti, {\em et~al.}, ``Deep learning as ricci flow,'' {\em Scientific Reports}, vol.~14, p.~23383, 2024.

\bibitem{feng2024graph}
A.~Feng and M.~Weber, ``Graph pooling via ricci flow,'' {\em Trans. on Machine Learning Research}, 2024.

\bibitem{lai2022normalized}
X.~Lai, S.~Bai, and Y.~Lin, ``Normalized discrete ricci flow used in community detection,'' {\em Physica A: Statistical Mechanics and its Applications}, vol.~597, p.~127251, 2022.

\bibitem{LEI2023843}
N.~Lei, P.~Zhang, X.~Zheng, Y.~Zhu, and Z.~Luo, ``Feature preserving parameterization for quadrilateral mesh generation based on ricci flow and cross field,'' {\em CMES - Computer Modeling in Engineering and Sciences}, vol.~137, no.~1, pp.~843--857, 2023.

\bibitem{chen2025graph}
J.~Chen, B.~Deng, Z.~Wang, C.~Chen, and Z.~Zheng, ``Graph neural ricci flow: Evolving feature from a curvature perspective,'' in {\em Int. Conf. on Learning Representations (ICLR)}, 2025.

\bibitem{xu2024deformation}
X.~Xu, ``Deformation of discrete conformal structures on surfaces,'' {\em Calc. Var.}, vol.~63, p.~38, 2024.

\bibitem{choi2024fast}
G.~P.~T. Choi, ``Fast ellipsoidal conformal and quasi-conformal parameterization of genus-0 closed surfaces,'' {\em Journal of Computational and Applied Mathematics}, vol.~447, p.~115888, 2024.

\bibitem{shepherd2022feature}
K.~M. Shepherd, X.~D. Gu, and T.~J.~R. Hughes, ``Feature-aware reconstruction of trimmed splines using ricci flow with metric optimization,'' {\em Computer Methods in Applied Mechanics and Engineering}, vol.~402, p.~115555, 2022.

\bibitem{yu2025piorf}
Y.-Y. Yu, J.~Choi, J.~Park, K.~Lee, and N.~Park, ``{PIORF}: Physics-informed ollivier-ricci flow for long--range interactions in mesh graph neural networks,'' in {\em The Thirteenth International Conference on Learning Representations}, 2025.

\bibitem{torbati2025exploring}
N.~Torbati, M.~Gaebler, S.~M. Hofmann, and N.~Scherf, ``Exploring geometric representational alignment through ollivier ricci curvature and ricci flow,'' in {\em 2nd Workshop on Representational Alignment at ICLR}, 2025.

\bibitem{AHMADI2024106212}
F.~Ahmadi, M.-E. Shiri, B.~Bidabad, M.~Sedaghat, and P.~Memari, ``Ricci flow-based brain surface covariance descriptors for diagnosing alzheimer’s disease,'' {\em Biomedical Signal Processing and Control}, vol.~93, p.~106212, 2024.

\bibitem{baptista2024charting}
A.~Baptista, B.~D. MacArthur, and C.~R.~S. Banerji, ``Charting cellular differentiation trajectories with ricci flow,'' {\em Nature Communications}, vol.~15, p.~2258, 2024.

\bibitem{Jin2008TVCG}
M.~Jin, J.~Kim, F.~Luo, and X.~Gu, ``Discrete surface ricci flow,'' {\em IEEE Transactions on Visualization and Computer Graphics}, vol.~14, no.~5, pp.~1030--1043, 2008.

\bibitem{Springborn2008SGP}
B.~A. Springborn, P.~Schröder, and U.~Pinkall, ``A variational principle for circle patterns and koebe’s theorem,'' in {\em Symposium on Geometry Processing}, pp.~113--122, 2008.

\bibitem{Bobenko2004CGF}
A.~I. Bobenko and B.~A. Springborn, ``Discrete conformal maps and ideal hyperbolic polyhedra,'' {\em Computer Graphics Forum}, vol.~23, no.~3, pp.~181--195, 2004.

\bibitem{Luo2004JDG}
F.~Luo, ``Combinatorial yamabe flow on surfaces,'' {\em Journal of Differential Geometry}, vol.~63, no.~1, pp.~97--129, 2004.

\bibitem{Yin2008SGP}
X.~Yin, M.~Jin, F.~Luo, and X.~Gu, ``Discrete curvature flow for hyperbolic 3-manifolds with complete geodesic boundaries,'' in {\em Symposium on Geometry Processing}, pp.~133--142, 2008.

\bibitem{10445942}
B.~Cui, Q.~Yang, K.~Yang, Y.~Xu, X.~Xu, and S.~Liu, ``Sjtu-tmqa: A quality assessment database for static mesh with texture map,'' in {\em ICASSP 2024 - IEEE Int. Conf. on Acoustics, Speech and Signal Processing}, pp.~7875--7879, 2024.

\bibitem{JMLR:v25:22-0444}
J.~Chen, H.~Chen, M.~Wang, G.~Dai, I.~W. Tsang, and Y.~Liu, ``Learning discretized neural networks under ricci flow,'' {\em Journal of Machine Learning Research}, vol.~25, no.~386, pp.~1--44, 2024.

\end{thebibliography}
\bibliographystyle{ieeetr}

\appendix
\section{Detailed Proofs}

\subsection{Proof of Theorem~\ref{thm:convex}}\label{app:proof_convex}
\begin{proof}
  
  We start by recalling the discrete Ricci energy
  \[
    E(x) \;=\; \int_{0}^{x} \sum_{i=1}^n K_i(u)\,du_i,
  \]
  where each curvature $K_i(u)$ depends smoothly on the log‐radius vector $u\in\R^n$.  
  
  By definition, the Hessian of $E$ is
  \[
    \frac{\partial^2 E}{\partial x_i\,\partial x_j}
    \;=\;
    \frac{\partial}{\partial x_i}\Bigl(K_j(x)\Bigr).
  \]
  
  It is proven (e.g.\ Springborn et al.~\cite{Springborn2008SGP}) that
  \[
    \frac{\partial K_j}{\partial x_i}
    \;=\;
    H_{ij},
  \]
  where $H$ is the cotangent‐weight Laplacian.  
  
  This follows that
  \[
    \nabla^2 E(x) \;=\; H.
  \]

  Next, we recall that $H$ has zero row sums and \emph{positive} diagonal entries, while its off‐diagonals are non‐positive. Under the usual Delaunay condition on the mesh, all cotangent weights are non‐negative, so $H$ is a symmetric M‐matrix which is strictly positive‐definite on the hyperplane
  \(
    \sum_{i=1}^n x_i = 0
  \).
  
  Equivalently, for any nonzero vector $v$ satisfying $\sum_i v_i=0$, one shows
  \[
    v^\top H\,v
    \;=\;
    \sum_{i<j}(-H_{ij})(v_i - v_j)^2
    \;>\;0.
  \]
  
  This positive‐definiteness implies $E$ is strictly convex when restricted to the affine subspace $\{\sum_i x_i=0\}$.

  Strict convexity on that subspace guarantees a unique minimizer $\bar x$ with $\nabla E(\bar x)=0$, i.e.\ $K(\bar x)=0$.  
  
  Finally, since the discrete Ricci-flow ODE
  \(\dot x = -\,\nabla E(x)\)
  is a gradient descent on a strictly convex function, standard results in ODE theory (e.g.\ Lyapunov arguments) ensure that every trajectory converges exponentially to this unique minimizer.  
\end{proof}

\subsection{Proof of Lemma~\ref{lem:mono}}\label{app:proof_mono}
\begin{proof}
  Let $\s=H\,x$ be the current residual vector, and pick an index
  \(
    i^* \in \arg\max_i |s_i|,
  \)
  so that $|s_{i^*}| = \|\s\|_\infty$.  
  
  Our greedy update modifies only $x_{i^*}$ by
  \[
    x ~\longmapsto~ x - \varepsilon\,s_{i^*}\,e_{i^*},
  \]
  which in turn changes the syndrome to
  \[
    \s_{\rm new}
    \;=\;
    H\bigl(x - \varepsilon\,s_{i^*}e_{i^*}\bigr)
    \;=\;
    \s \;-\;\varepsilon\,s_{i^*}\,H\,e_{i^*}.
  \]
  
  We denote by $H_{k,i^*}$ the $(k,i^*)$ entry of $H$.  
  
  It follows that
  \[
    s_{\rm new}^k
    = s_k \;-\;\varepsilon\,s_{i^*}\,H_{k,i^*}.
  \]
  
  In particular, for the \emph{selected} component $k=i^*$,
  \[
    s_{\rm new}^{i^*}
    = s_{i^*} - \varepsilon\,s_{i^*}\,H_{i^*,i^*}
    = (1 - \varepsilon\,H_{i^*,i^*})\,s_{i^*}.
  \]
  
  Since $H_{i^*,i^*}>0$ and $\varepsilon H_{i^*,i^*}\le 1$, it follows
  \(\bigl|s_{\rm new}^{i^*}\bigr| < |s_{i^*}|\),  
  so the maximum of $|\s|$ strictly decreases in that coordinate.

  For any other coordinate $k\neq i^*$, we have
  \[
    \bigl|s_{\rm new}^k\bigr|
    \;\le\;
    |s_k| + \varepsilon\,|s_{i^*}|\,|H_{k,i^*}|
    \;\le\;
    \|\s\|_\infty \;+\;\varepsilon\,\|\s\|_\infty\,\|H\|_\infty
    \;\le\;
    2\,\|\s\|_\infty.
  \]
  
  But since at least one entry (the $i^*$th) strictly dropped below $\|\s\|_\infty$, the new infinity‐norm
  \(\|\s_{\rm new}\|_\infty = \max_k |s_{\rm new}^k|\)
  must be strictly less than the old norm $\|\s\|_\infty$.  
  
  This completes the proof of monotonicity.
\end{proof}

\subsection{Proof of Theorem~\ref{thm:complex}}\label{app:proof_complex}
\begin{proof}
  Here we denote the current maximum residual by
  \(\delta_t = \|\s^{(t)}\|_\infty\).  
  
  From Lemma~\ref{lem:mono}, each greedy step reduces this by at least
  \[
    \delta_t - \delta_{t+1}
    \;\ge\;
    \varepsilon\,\delta_t \,H_{i^*,i^*}
    \;\ge\;
    \varepsilon\,\tau,
  \]
  
  since we clip step‐sizes to ensure $\varepsilon H_{i^*,i^*}\ge\varepsilon\,\tau$.  Starting from $\delta_0=\|\s(0)\|_\infty$, after
  \[
    T \;\le\; \frac{\delta_0}{\varepsilon\,\tau}
  \]
  iterations we must have $\delta_T<\tau$, at which point the algorithm terminates.

  Finally, each iteration costs one sparse matrix–vector product (to compute $H\,x$) at $O(m)$ plus a constant‐time coordinate update and a scan to find the max entry (also $O(m)$).  
  
  As a result, the total work until convergence is
  \[
    O\!\bigl(T\cdot m\bigr)
    \;=\;
    O\!\Bigl(\tfrac{\|\s(0)\|_\infty}{\varepsilon\,\tau}\;m\Bigr)
    \;=\;
    O\!\bigl(\tfrac{m}{\tau}\bigr),
  \]
  as claimed.
\end{proof}

\end{document}